\documentclass[letterpaper]{article}

\usepackage{aaai}
\usepackage{times}
\usepackage{helvet}
\usepackage{courier}
\usepackage{amsmath}
\usepackage{amsthm}
\usepackage{color}
\usepackage[utf8]{inputenc}
\usepackage{graphicx}
\usepackage{dsfont}
\usepackage{algorithmicx}
\usepackage{algpseudocode}
\usepackage{xspace}
\usepackage{hyperref}

\title{Reinforcement Learning in POMDPs with Memoryless Options and Option-Observation Initiation Sets}
\author{
	Denis Steckelmacher \\
	{\bf \large Diederik M. Roijers} \\
	{\bf \large Anna Harutyunyan} \\
	{\bf \large Peter Vrancx} \\
	{\bf \large Hélène Plisnier} \\
	{\bf \large Ann Nowé} \\
	Artificial Intelligence Lab, Vrije Universiteit Brussel, Belgium
}

\nocopyright
\allowdisplaybreaks
\newtheorem{theorem}{Theorem}

\newcommand{\N}{\ensuremath{\mathcal{N}}}

\newcommand{\Next}{OOIs\xspace}
\newcommand{\md}{{-}}

\begin{document}
	
	\maketitle
	
	\begin{abstract}
		Many real-world reinforcement learning problems have a hierarchical nature, and often exhibit some degree of partial observability. While hierarchy and partial observability are usually tackled separately (for instance by combining recurrent neural networks and options), we show that addressing both problems simultaneously is simpler and more efficient in many cases. More specifically, we make the initiation set of options conditional on the previously-executed option, and show that options with such Option-Observation Initiation Sets (\Next) are at least as expressive as Finite State Controllers (FSCs), a state-of-the-art approach for learning in POMDPs. \Next are easy to design based on an intuitive description of the task, lead to explainable policies and keep the top-level and option policies memoryless. Our experiments show that \Next allow agents to learn optimal policies in challenging POMDPs, while being much more sample-efficient than a recurrent neural network over options.
	\end{abstract}
	
	\section{Introduction}
	
	Real-world applications of reinforcement learning (RL) face two main challenges: complex long-running tasks and partial observability. Options, the particular instance of Hierarchical RL we focus on, addresses the first challenge by factoring a complex task into simpler sub-tasks \cite{Barto2003,Roy2006,Tessler2016}. Instead of learning what action to perform depending on an observation, the agent learns a top-level policy that repeatedly selects options, that in turn execute sequences of actions before returning \cite{Sutton1999}. The second challenge, partial observability, is addressed by maintaining a belief of what the agent thinks the full state is \cite{Kaelbling1998,Cassandra1994}, reasoning about possible future observations \cite{Littman2001,Boots2009}, storing information in an external memory for later reuse \cite{Peshkin2001,Zaremba2015,Graves2016}, or using recurrent neural networks (RNNs) to allow information to flow between time-steps \cite{Bakker2001,Mnih2016}.
	
	Combined solutions to the above two challenges have recently been designed for planning \cite{He2011}, but solutions for learning algorithms are not yet ideal. HQ-Learning decomposes a task into a sequence of fully-observable subtasks \cite{Wiering1997}, which precludes cyclic tasks from being solved. Using recurrent neural networks in options and for the top-level policy \cite{Sridharan2010} addresses both challenges, but brings in the design complexity of RNNs \cite{Jozefowicz15,Angeline94,Mikolov14}. RNNs also have limitations regarding long time horizons, as their memory decays over time \cite{Hochreiter1997}.
	
	In her PhD thesis, Precup (\citeyear{Precup2000}, page 126) suggests that options may already be close to addressing partial observability, thus removing the need for more complicated solutions. In this paper, we prove this intuition correct by:
	
	\begin{enumerate}
		\item Showing that standard options do not suffice in POMDPs;
		\item Introducing Option-Observation Initiation Sets (\Next), that make the initiation sets of options conditional on the previously-executed option;
		\item Proving that \Next make options at least as expressive as Finite State Controllers (Section \ref{sec:fsc}), thus able to tackle challenging POMDPs.
	\end{enumerate}
	
	\noindent
	In contrast to existing HRL algorithms for POMDPs \cite{Wiering1997,Theocharous2002,Sridharan2010}, \Next handle repetitive tasks, do not restrict the action set available to sub-tasks, and keep the top-level and option policies memoryless. A wide range of robotic and simulated experiments in Section \ref{sec:experiments} confirm that \Next allow partially observable tasks to be solved optimally, demonstrate that \Next are much more sample-efficient than a recurrent neural network over options, and illustrate the flexibility of \Next regarding the amount of domain knowledge available at design time. In Section \ref{sec:treemaze}, we demonstrate the robustness of \Next to sub-optimal option sets. While it is generally accepted that the designer provides the options and their initiation sets, we show in Section \ref{sec:duplicatedinput} that random initiation sets, combined with learned option policies and termination functions, allow \Next to be used without any domain knowledge.
	
	\subsection{Motivating Example}
	\label{sec:motivation}
	
	\begin{figure}
		\centering
		\includegraphics{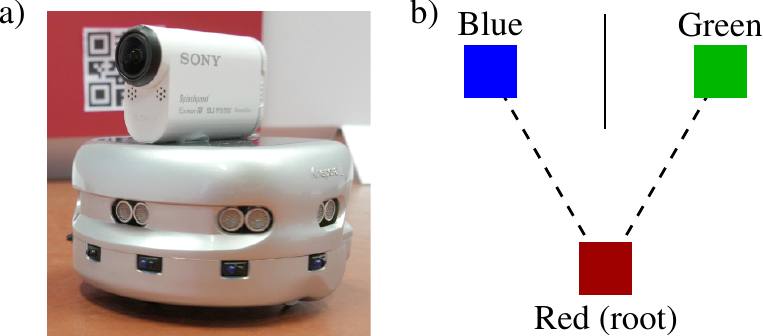}
		\caption{Robotic object gathering task. a) Khepera III, the two-wheeled robot used in the experiments. b) The robot has to gather objects from two terminals separated by a wall, and to bring them to the root.}
		\label{fig:khepera}
	\end{figure}
	
	\Next are designed to solve complex partially-observable tasks that can be decomposed into a set of fully-observable sub-tasks. For instance, a robot with first-person sensors may be able to avoid obstacles, open doors or manipulate objects even if its precise location in the building is not observed. We now introduce such an environment, on which our robotic experiments of Section \ref{sec:khepera} are based.
	
	A Khepera III robot\footnotemark has to gather objects from two terminals separated by a wall, and to bring them to the root (see Figure \ref{fig:khepera}). Objects have to be gathered one by one from a terminal until it becomes empty, which requires many journeys between the root and a terminal. When a terminal is emptied, the other one is automatically refilled. The robot therefore has to alternatively gather objects from both terminals, and the episode finishes after the terminals have been emptied some random number of times. The root is colored in red and marked by a paper QR-code encoding \texttt{1}. Each terminal has a screen displaying its color and a dynamic QR-code (\texttt{1} when full, \texttt{2} when empty). Because the robot cannot read QR-codes from far away, the state of a terminal cannot be observed from the root, where the agent has to decide to which terminal it will go. This makes the environment partially observable, and requires the robot to remember which terminal was last visited, and whether it was full or empty.
	
	The robot is able to control the speed of its two wheels. A wireless camera mounted on top of the robot detects bright color blobs in its field of view, and can read nearby QR-codes. Such low-level actions and observations, combined with a complicated task, motivate the use of hierarchical reinforcement learning. Fixed options allow the robot to move towards the largest red, green or blue blob in its field of view. The options terminate as soon as a QR-code is in front of the camera and close enough to be read. The robot has to learn a policy over options that solves the task.
	
	\footnotetext{\url{http://www.k-team.com/mobile-robotics-products/old-products/khepera-iii}}
	
	The robot may have to gather a large number of objects, alternating between terminals several times. The repetitive nature of this task is incompatible with HQ-Learning \cite{Wiering1997}. Options with standard initiation sets are not able to solve this task, as the top-level policy is memoryless \cite{Sutton1999} and cannot remember from which terminal the robot arrives at the root, and whether that terminal was full or empty. Because the terminals are a dozen feet away from the root, almost a hundred primitive actions have to be executed to complete any root/terminal journey. Without options, this represents a time horizon much larger than usually handled by recurrent neural networks \cite{Bakker2001} or finite history windows \cite{Lin1993}.
	
	\Next allow each option to be selected conditionally on the previously executed one (see Section \ref{sec:theory}), which is much simpler than combining options and recurrent neural networks \cite{Sridharan2010}. The ability of \Next to solve complex POMDPs builds on the time abstraction capabilities and expressiveness of options. Section \ref{sec:khepera} shows that \Next allow a policy for our robotic task to be learned to expert level. Additional experiments demonstrate that both the top-level and option policies can be learned by the agent (see Section \ref{sec:duplicatedinput}), and that \Next lead to substantial gains over standard initiation sets even if the option set is reduced or unsuited to the task (see Section \ref{sec:treemaze}).
	
	\begin{figure}
		\centering
		\includegraphics{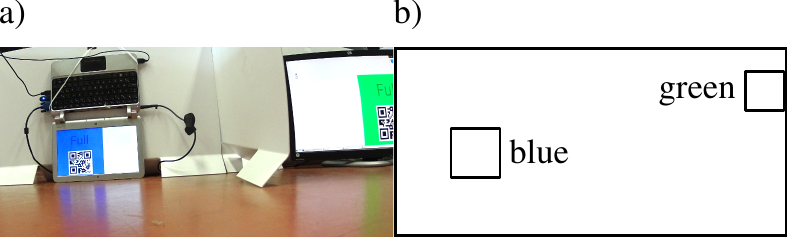}
		\caption{Observations of the Khepera robot. a) Color image from the camera. b) Color blobs detected by the vision system, as observed by the robot. QR-codes can only be decoded when the robot is a couple of inches away from them.}
		\label{fig:khepera_observations}
	\end{figure}
	
	\section{Background}
	
	This section formally introduces Markov Decision Processes (MDPs), Options, Partially Observable MDPs (POMDPs) and Finite State Controllers, before presenting our main contribution in Section \ref{sec:ooi}.
	
	\subsection{Markov Decision Processes}
	
	A discrete-time Markov Decision Process (MDP) $\langle S, A, R, T, \gamma \rangle$ with discrete actions is defined by a possibly-infinite set $S$ of states, a finite set $A$ of actions, a reward function $R(s_t, a_t, s_{t+1}) \in \mathcal{R}$, that provides a scalar reward $r_t$ for each state transition, a transition function $T(s_t, a_t, s_{t+1}) \in [0, 1]$, that outputs a probability distribution over new states $s_{t+1}$ given a $(s_t, a_t)$ state-action pair, and $0 \le \gamma < 1$ the discount factor, that defines how sensitive the agent should be to future rewards.
	
	A stochastic memoryless policy $\pi(s_t, a_t) \in [0, 1]$ maps a state to a probability distribution over actions. The goal of the agent is to find a policy $\pi^*$ that maximizes the expected cumulative discounted reward $E_{\pi^*}[\sum_t \gamma^t r_t]$ obtainable by following that policy.
	
	\subsection{Options}
	
	The options framework, defined in the context of MDPs \cite{Sutton1999}, consists of a set of options $O$ where each option $\omega \in O$ is a tuple $\langle \pi_\omega, I_\omega, \beta_\omega \rangle$, with $\pi_\omega (s_t, a_t) \in [0, 1]$ the memoryless option policy, $\beta_\omega(s_t) \in [0, 1]$ the termination function that gives the probability for the option $\omega$ to terminate in state $s_t$, and $I_\omega \subseteq S$ the initiation set that defines in which states $\omega$ can be started \cite{Sutton1999}.
	
	The memoryless top-level policy $\mu(s_t, \omega_t) \in [0, 1]$ maps states to a distribution over options and allows to choose which option to start in a given state. When an option $\omega$ is started, it executes until termination (due to $\beta_\omega$), at which point $\mu$ selects a new option based on the now current state.
	
	\subsection{Partially Observable MDPs}
	
	Most real-world problems are not completely captured by MDPs, and exhibit at least some degree of partial observability. A Partially Observable MDP (POMDP) $\langle \Omega, S, A, R, T, W, \gamma \rangle$ is an MDP extended with two components: the possibly-infinite set $\Omega$ of observations, and the $W : S \rightarrow \Omega$ function that produces observations $x$ based on the unobservable state $s$ of the process. Two different states, requiring two different optimal actions, may produce the same observation. This makes POMDPs remarkably challenging for reinforcement learning algorithms, as memoryless policies, that select actions or options based only on the current observation, typically no longer suffice.
	
	\subsection{Finite State Controllers}
	
	Finite State Controllers (FSCs) are commonly used in POMDPs. An FSC $\langle \N, \psi, \eta, \eta^0 \rangle$ is defined by a finite set \N~of nodes, an action function $\psi (n_t, a_t) \in [0, 1]$ that maps nodes to a probability distribution over actions, a successor function $\eta (n_{t-1}, x_t, n_t) \in [0, 1]$ that maps nodes and observations to a probability distribution over next nodes, and an initial function $\eta^0 (x_1, n_1) \in [0, 1]$ that maps initial observations to nodes \cite{Meuleau1999}.

	At the first time-step, the agent observes $x_1$ and activates a node $n_1$ by sampling from $\eta^0 (x_1, \cdot)$. An action is performed by sampling from $\psi (n_1, \cdot)$. At each time-step $t$, a node $n_t$ is sampled from $\eta (n_{t-1}, x_t, \cdot)$, then an action $a_t$ is sampled from $\psi (n_t, \cdot)$. FSCs allow the agent to select actions according to the entire history of past observations \cite{Meuleau1999}, which has been shown to be one of the best approaches for POMDPs \cite{Lin1992}. \Next, our main contribution, make options at least as expressive and as relevant to POMDPs as FSCs, while being able to leverage the hierarchical structure of the problem.
	
	\section{Option-Observation Initiation Sets}
	\label{sec:ooi}
	
	Our main contribution, Option-Observation Initiation Sets (\Next), make the initiation sets of options conditional on the option that has just terminated. We prove that \Next make options at least as expressive as FSCs (thus suited to POMDPs, see Section \ref{sec:fsc}), even if the top-level and option policies are memoryless, while options without \Next are strictly less expressive than FSCs (see Section \ref{sec:vanilla}). In Section \ref{sec:experiments}, we show on one robotic and two simulated tasks that \Next allow challenging POMDPs to be solved optimally.
	
	\subsection{Conditioning on Previous Option}
	\label{sec:theory}
	\newcommand{\Ot}{\ensuremath{\mathcal{O}_t}}
	
	Descriptions of partially observable tasks in natural language often contain allusions at sub-tasks that must be sequenced or cycled through, possibly with branches. This is easily mapped to a policy over options (learned by the agent) and sets of options that may or may not follow each other.
	
	A good memory-based policy for our motivating example, where the agent has to bring objects from two terminals to the root, can be described as ``go to the green terminal, then go to the root, then go back to the green terminal if it was full, to the blue terminal otherwise'', and symmetrically so for the blue terminal. This sequence of sub-tasks, that contains a condition, is easily translated to a set of options. Two options, $\omega_{GF}$ and $\omega_{GE}$, sharing a single policy, go from the green terminal to the root (using low-level motor actions). $\omega_{GF}$ is executed when the terminal is full, $\omega_{GE}$ when it is empty. At the root, the option that goes back to the green terminal can only follow $\omega_{GF}$, not $\omega_{GE}$. When the green terminal is empty, going back to it is therefore forbidden, which forces the agent to switch to the blue terminal when the green one is empty.
	
	We now formally define our main contribution, Option-Observation Initiation Sets (\Next), that allow to describe which options may follow which ones. We define the initiation set $I_\omega$ of option $\omega$ so that the set \Ot~of options available at time $t$ depends on the observation $x_t$ and previously-executed option $\omega_{t-1}$:
	
	\begin{align*}
		I_\omega &\subseteq \Omega \times (O \cup \{ \emptyset \}) \\
		\Ot &\equiv \{ \omega \in O : (x_t, \omega_{t-1}) \in I_\omega \}
	\end{align*}
	
	\noindent
	with $\omega_0 = \emptyset$, $\Omega$ the set of observations and $O$ the set of options. \Ot~allows the agent to condition the option selected at time $t$ on the one that has just terminated, even if the top-level policy does not observe $\omega_{t-1}$. The top-level and option policies remain memoryless. Not having to observe $\omega_{t-1}$ keeps the observation space of the top-level policy small, instead of extending it to $\Omega \times O$, without impairing the representational power of \Next, as shown in the next sub-section.
	
	\subsection{\Next Make Options as Expressive as FSCs}
	\label{sec:fsc}
	
	Finite State Controllers are state-of-the-art in policies applicable to POMDPs \cite{Meuleau1999}. By proving that options with \Next are as expressive as FSCs, we provide a lower bound on the expressiveness of \Next and ensure that they are applicable to a wide range of POMDPs.
	
	\begin{theorem}
		\label{thm:fsc}
		\Next allow options to represent any policy that can be expressed using a Finite State Controller.
	\end{theorem}
	
	\begin{proof}
		\newcommand{\opt}{\ensuremath{\langle n'_{t-1}, n_t \rangle}}
		\newcommand{\optprev}{\ensuremath{\langle n'_{t-2}, n_{t-1} \rangle}}
		\newcommand{\initopt}{\ensuremath{\langle \emptyset, n_1 \rangle}}
		
		The reduction from any FSC to options requires one option \opt~per ordered pair of nodes in the FSC, and one option \initopt~per node in the FSC. Assuming that $n_0 = \emptyset$ and $\eta (\emptyset, x_1, \cdot) = \eta^0 (x_1, \cdot)$, the options are defined by:
		
		\begin{align}
			\label{eq:fsc_beta}
			\beta_{\opt} (x_t)     &= 1 \\
			\label{eq:fsc_pio}
			\pi_{\opt} (x_t, a_t)  &= \psi(n_t, a_t) \\
			\label{eq:fsc_pi}
			\mu (x_t, \opt)        &= \eta (n'_{t-1}, x_t, n_t) \\
			\nonumber
			I_{\initopt}           &= \Omega \times \{ \emptyset \} \\
			\nonumber
			I_{\opt}               &= \Omega \times \{\optprev : n'_{t-1} = n_{t-1} \}
		\end{align}
		
		Each option corresponds to an edge of the FSC. Equation \ref{eq:fsc_beta} ensures that every option stops after having emitted a single action, as the FSC takes one transition every time-step. Equation \ref{eq:fsc_pio} maps the current option to the action emitted by the destination node of its corresponding FSC edge. We show that $\mu$ and $I_{\opt}$ implement $\eta (n_{t-1}, x_t, n_t)$, with $\omega_{t-1} = \optprev$, by:
		
		\begin{align*}
			\mu   &(x_t, \opt) = \\ &\begin{cases}
				\eta (n_{t-1}, x_t, n_t) &\begin{array}{l}
					\optprev \in I_{\opt} \\
					\Leftrightarrow n'_{t-1} = n_{t-1}
				\end{array} \\
				0 &\begin{array}{l}
					\optprev \notin I_{\opt} \\
					\Leftrightarrow n'_{t-1} \ne n_{t-1}
				\end{array}
			\end{cases}
		\end{align*}
		
		Because $\eta$ maps nodes to nodes and $\mu$ selects options representing pairs of nodes, $\mu$ is extremely sparse and returns a value different from zero, $\eta (n_{t-1}, x_t, n_t)$, only when \optprev~and \opt~agree on $n_{t-1}$.
	\end{proof}

	Our reduction uses options with trivial policies, that execute for a single time-step, which leads to a large amount of options to compensate. In practice, we expect to be able to express policies for real-world POMDPs with much less options than the number of states an FSC would require, as shown in our simulated (Section \ref{sec:duplicatedinput}, 2 options) and robotic experiments (Section \ref{sec:khepera}, 12 options). In addition to being sufficient, the next sub-section proves that \Next are necessary for options to be as expressive as FSCs.
	
	\subsection{Original Options are not as Expressive as FSCs}
	\label{sec:vanilla}
	
	\begin{figure}
		\centering
		\includegraphics{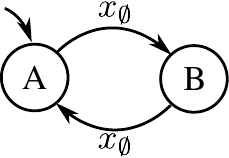}
		\caption{Two-nodes Finite State Controller that emits an infinite sequence ABAB... based on an uninformative observation $x_{\emptyset}$. This FSC cannot be expressed using options without \Next.}
		\label{fig:fsc}
	\end{figure}
	
	While options with regular initiation sets are able to express some memory-based policies \cite[page 7]{Sutton1999}, the tiny but valid Finite State Controller presented in Figure \ref{fig:fsc} cannot be mapped to a set of options and a policy over options (without \Next). This proves that options without \Next are strictly less expressive than FSCs.
	
	\begin{theorem}
		Options without \Next are not as expressive as Finite State Controllers.
	\end{theorem}

	\begin{proof}
		Figure \ref{fig:fsc} shows a Finite State Controller that emits a sequence of alternating A's and B's, based on a constant uninformative observation $x_{\emptyset}$. This task requires memory because the observation does not provide any information about what was the last letter to be emitted, or which one must now be emitted. Options having memoryless policies, options executing for multiple time-steps are unable to represent the FSC exactly. A combination of options that execute for a single time-step cannot represent the FSC either, as the options framework is unable to represent memory-based policies with single-time-step options \cite{Sutton1999}.
	\end{proof}
	
	\section{Experiments}
	\label{sec:experiments}
	
	The experiments in this section illustrate how \Next allow agents to perform optimally in environments where options without \Next fail. Section \ref{sec:khepera} shows that \Next allow the agent to learn an expert-level policy for our motivating example (Section \ref{sec:motivation}). Section \ref{sec:duplicatedinput} shows that the top-level and option policies required by a repetitive task can be learned, and that learning option policies allow the agent to leverage random \Next, thereby removing the need for designing them. In Section \ref{sec:treemaze}, we progressively reduce the amount of options available to the agent, and demonstrate how \Next still allow good memory-based policies to emerge when a sub-optimal amount of options are used.
	
	All our results are averaged over 20 runs, with standard deviation represented by the light regions in the figures. The source code, raw experimental data, run scripts, and plotting scripts of our experiments, along with a detailed description of our robotic setup, are available as supplementary material. A video detailing our robotic experiment is available at \url{http://steckdenis.be/oois_demo.mp4}.
	
	\subsection{Learning Algorithm}
	\label{sec:nnet}
	
	All our agents learn their top-level and option policies (if not provided) using a single feed-forward neural network, with one hidden layer of 100 neurons, trained using Policy Gradient \cite{Sutton2000} and the Adam optimizer \cite{kingma2014adam}. Our neural network $\pi$ takes three inputs and produces one output. The inputs are problem-specific observation features $\mathbf{x}$, the one-hot encoded current option $\boldsymbol{\omega}$ ($\boldsymbol{\omega} = \mathbf{0}$ when executing the top-level policy), and a mask, $\mathbf{mask}$. The output $\mathbf{y}$ is the joint probability distribution over selecting actions or options (so that the same network can be used for the top-level and option policies), while terminating or continuing the current option:
	
	\begin{align*}
		\mathbf{h_1} &= \tanh(\mathbf{W}_1 [\mathbf{x}^T \boldsymbol{\omega}^T]^T + \mathbf{b}_1), \\
		\mathbf{\hat{y}} &= \sigma(\mathbf{W}_2 \mathbf{h_1} + \mathbf{b}_2) \circ \mathbf{mask}, \\
		\mathbf{y} &= \frac{\mathbf{\hat{y}}}{\mathbf{1}^T \mathbf{\hat{y}}},
	\end{align*}
	
	\noindent
	with $W_i$ and $b_i$ the trainable weights and biases of layer $i$, $\sigma$ the sigmoid function, and $\circ$ the element-wise product of two vectors. The fraction ensures that a valid probability distribution is produced by the network. The initiation sets of options are implemented using the $\mathbf{mask}$ input of the neural network, a vector of $2 \times (|A| + |O|)$ integers, the same dimension as the $\mathbf{y}$ output. When executing the top-level policy ($\boldsymbol{\omega} = \mathbf{0}$), the mask forces the probability of primitive actions to zero, preserves option $\omega_i$ according to $I_{\omega_i}$, and prevents the top-level policy from terminating. When executing an option policy ($\boldsymbol{\omega} \ne \mathbf{0}$), the mask only allows primitive actions to be executed. For instance, if there are two options and three actions,
	$\mathbf{mask} = \begin{smallmatrix}end \\ cont\end{smallmatrix} (
		\begin{smallmatrix}
			0 & 0 & 1 & 1 & 1 \\
			0 & 0 & 1 & 1 & 1
		\end{smallmatrix}
	)$
	when executing any of the options. When executing the top-level policy,
	$\mathbf{mask} = \begin{smallmatrix}end \\ cont\end{smallmatrix} (
		\begin{smallmatrix}
			0 & 0 & 0 & 0 & 0 \\
			a & b & 0 & 0 & 0
		\end{smallmatrix}
	)$,
	with $a = 1$ if and only if the option that has just finished is in the initiation set of the first option, and $b = 1$ according to the same rule but for the second option. The neural network $\pi$ is trained using Policy Gradient, with the following loss:
	
	\begin{align*}
		\mathcal{L}(\pi) &= -\sum\limits_{t=0}^{T} (\mathcal{R}_t - V(x_t, \omega_t)) \log (\pi(x_t, \omega_t, a_t))
	\end{align*}
	
	\noindent
	with $a_t \sim \pi(x_t, \omega_t, \cdot)$ the action executed at time $t$. The return $\mathcal{R}_t = \sum_{\tau=t}^{T} \gamma^{\tau} r_{\tau}$, with $r_{\tau} = R(s_{\tau}, a_{\tau}, s_{\tau+1})$, is a simple discounted sum of future rewards, and ignores changes of current option. This gives the agent information about the complete outcome of an action or option, by directly evaluating its flattened policy. A baseline $V(x_t, \omega_t)$ is used to reduce the variance of the $\mathcal{L}$ estimate \cite{Sutton2000}. $V(x_t, \omega_t)$ predicts the expected cumulative reward obtainable from $x_t$ in option $\omega_t$ using a separate neural network, trained on the monte-carlo return obtained from $x_t$ in $\omega_t$.
	
	\subsection{Comparison with LSTM over Options}
	\label{sec:lstm}

	In order to provide a complete evaluation of \Next, a variant of the $\pi$ and $V$ networks of Section \ref{sec:nnet}, where the hidden layer is replaced with a layer of 20 LSTM units \cite{Hochreiter1997,Sridharan2010}, is also evaluated on every task. We use 20 units as this leads to the best results in our experiments, which ensures a fair comparison of LSTM against OOIs. In all experiments, the LSTM agents are provided the same set of options as the agent with \Next. Not providing any option, or less options, leads to worse results. Options allow the LSTM network to focus on important observations, and reduces the time horizon to be considered. Shorter time horizons have been shown to be beneficial to LSTM \cite{Bakker2001}.
	
	Despite our efforts, LSTM over options only manages to learn good policies in our robotic experiment (see Section \ref{sec:khepera}), and requires more than twice the amount of episodes as \Next to do so. In our repetitive task, dozens of repetitions seem to confuse the network, that quickly diverges from any good policy it may learn (see Section \ref{sec:duplicatedinput}). On TreeMaze, a much more complex version of the T-maze task, originally used to benchmark reinforcement learning LSTM agents \cite{Bakker2001}, the LSTM agent learns the optimal policy after more than 100K episodes (not shown on the figures). These results illustrate how learning with recurrent neural networks is sometimes difficult, and how \Next allow to reliably obtain good results, with minimal engineering effort.
	
	\subsection{Object Gathering}
	\label{sec:khepera}
	
	The first experiment illustrates how \Next allow an expert-level policy to be learned for a complex robotic partially-observable repetitive task.	The experiment takes place in the environment described in Section \ref{sec:motivation}. A robot has to gather objects one by one from two terminals, green and blue, and bring them back to the root location. Because our actual robot has no effector, it navigates between the root and the terminals, but only pretends to move objects. The agent receives a reward of +2 when it reaches a full terminal, -2 when the terminal is empty. At the beginning of the episode, each terminal contains 2 to 4 objects, this amount being selected randomly for each terminal. When the agent goes to an empty terminal, the other one is re-filled with 2 to 4 objects. The episode ends after 2 or 3 emptyings (combined across both terminals). Whether a terminal is full or empty is observed by the agent only when it is at the terminal. The agent therefore has to remember information acquired at terminals in order to properly choose, at the root, to which terminal it will go.
	
	\begin{figure}
		\centering
		\includegraphics{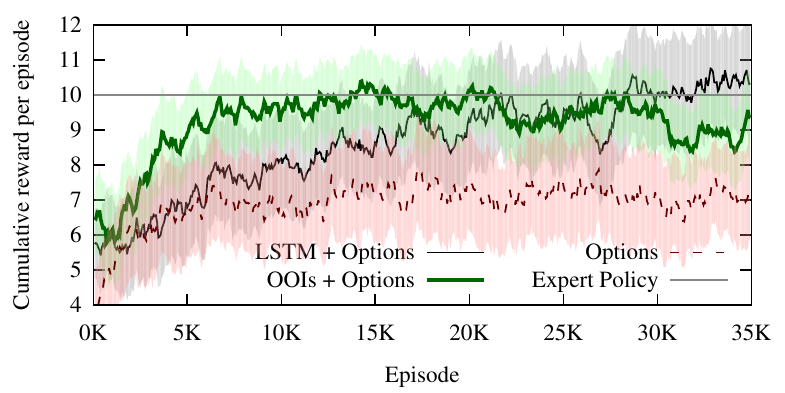}
		\caption{Cumulative reward per episode obtained on our object gathering task, with \Next, without \Next, and using an LSTM over options. \Next learns an expert-level policy much quicker than an LSTM over options. The LSTM curve flattens-out (with high variance) after about 30K episodes.}
		\label{fig:kresults}
	\end{figure}
	
	The agent has access to 12 memoryless options that go to red ($\omega_{R1..R4}$), green ($\omega_{G1..G4}$) or blue objects ($\omega_{B1..B4}$), and terminate when the agent is close enough to them to read a QR-code displayed on them. The initiation set of $\omega_{R1,R2}$ is $\omega_{G1..G4}$, of $\omega_{R3,R4}$ is $\omega_{B1..B4}$, and of $\omega_{G_i,B_i}$ is $\omega_{R_i} ~ \forall i = 1..4$. This description of the options and their \Next is purposefully uninformative, and illustrates how little information the agent has about the task. The option set used in this experiment is also richer than the simple example of Section \ref{sec:theory}, so that the solution of the problem, not going back to an empty terminal, is not encoded in \Next but must be learned by the agent.
	
	Agents with and without \Next learn top-level policies over these options. We compare them to a \emph{fixed} agent, using an expert top-level policy that interprets the options as follows: $\omega_{R1..R4}$ go to the root from a full/empty green/blue terminal (and are selected accordingly at the terminals depending on the QR-code displayed on them), while $\omega_{G1..G4,B1..B4}$ go to the green/blue terminal from the root when the previous terminal was full/empty and green/blue. At the root, \Next ensure that only one option amongst \emph{go to green after a full green}, \emph{go to green after an empty blue}, \emph{go to blue after a full blue} and \emph{go to blue after an empty green} is selected by the top-level policy: the one that corresponds to what color the last terminal was and whether it was full or empty. The agent goes to a terminal until it is empty, then switches to the other terminal, leading to an average reward of 10.\footnotemark
	
	When the top-level policy is learned, \Next allow the task to be solved, as shown in Figure \ref{fig:kresults}, while standard initiation sets do not allow the task to be learned. Because experiments on a robot are slow, we developed a small simulator for this task, and used it to produce Figure \ref{fig:kresults} after having successfully asserted its accuracy using two 1000-episodes runs on the actual robot. The agent learns to properly select options at the terminals, depending on the QR-code, and to output a proper distribution over options at the root, thereby matching our expert policy. The LSTM agent learns the policy too, but requires more than twice the amount of episodes to do so. The high variance displayed in Figure \ref{fig:kresults} comes from the varying amounts of objects in the terminals, and the random selection of how many times they have to be emptied.
	
	Because fixed option policies are not always available, we now show that \Next allow them to be learned at the same time as the top-level policy.
	
	\footnotetext{$\frac{2+3}{2} \times (-2 + \frac{2+4}{2} \times 2)$, 2 or 3 emptyings of terminals that contain 2 to 4 objects. Average confirmed experimentally from 1000 episodes using the policy, $p > 0.30$.}
	
	\subsection{Modified DuplicatedInput}
	\label{sec:duplicatedinput}
	
	In some cases, a hierarchical reinforcement learning agent may not have been provided policies for several or any of its options. In this case, \Next allow the agent to learn its top-level policy, the option policies and their termination functions. In this experiment, the agent has to learn its top-level and option policies to copy characters from an input tape to an output tape, removing duplicate B's and D's (mapping ABBCCEDD to ABCCED for instance; B's and D's always appear in pairs). The agent only observes a single input character at a time, and can write at most one character to the output tape per time-step.
	
	The input tape is a sequence of $N$ symbols $x \in \Omega$, with $\Omega = \{A, B, C, D, E\}$ and $N$ a random number between 20 and 30. The agent observes a single symbol $x_t \in \Omega$, read from the $i$-th position in the input sequence, and does not observe $i$. When $t = 1$, $i = 0$.  There are 20 actions ($5 \times 2 \times 2$), each of them representing a symbol (5), whether it must be pushed onto the output tape (2), and whether $i$ should be incremented or decremented (2). A reward of 1 is given for each correct symbol written to the output tape. The episode finishes with a reward of -0.5 when an incorrect symbol is written.
	
	\begin{figure}
		\centering
		\includegraphics{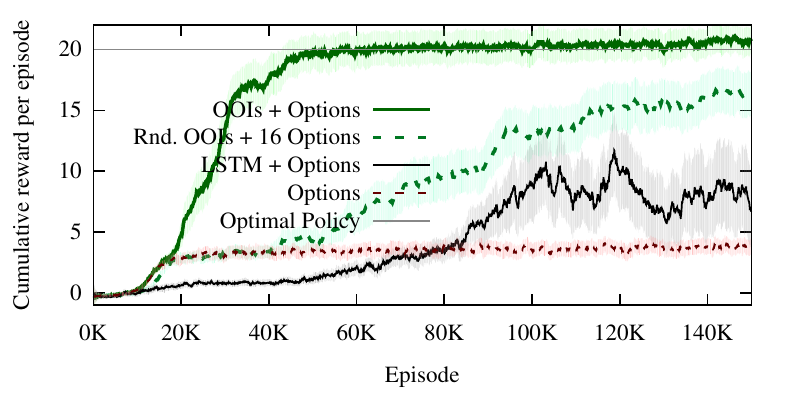}
		\caption{Cumulative reward per episode obtained on modified DuplicatedInput, with random or designed \Next, without \Next and using an LSTM over options. Despite our efforts, an LSTM over options repeatedly learns then forgets optimal policies, as shown by the high variance of its line.}
		\label{fig:diresults}
	\end{figure}
	
	The agent has access to two options, $\omega_1$ and $\omega_2$. \Next are designed so that $\omega_2$ cannot follow itself, with no such restriction on $\omega_1$. No reward shaping or hint about what each option should do is provided. The agent automatically discovers that $\omega_1$ must copy the current character to the output, and that $\omega_2$ must skip the character without copying it. It also learns the top-level policy, that selects $\omega_2$ (skip) when observing B or D and $\omega_2$ is allowed, $\omega_1$ otherwise (copy).
	
	Figure \ref{fig:diresults} shows that an agent with two options and \Next learns the optimal policy for this task, while an agent with two options and only standard initiation sets ($I_\omega = \Omega ~ \forall \omega$) fails to do so. The agent without \Next only learns to copy characters and never skips any (having two options does not help it). This shows that \Next are necessary for learning this task, and allow to learn top-level and option policies suited to our repetitive partially observable task.
	
	When the option policies are learned, the agent becomes able to adapt itself to random \Next, thereby removing the need for designing \Next. For an agent with $N$ options, each option has $\frac{N}{2}$ randomly-selected options in its initiation set, with the initiation sets re-sampled for each run. The agents learn how to leverage their option set, and achieve good results on average (16 options used in Figure \ref{fig:diresults}, more options lead to better results). When looking at individual runs, random \Next allow optimal policies to be learned, but several runs require more time than others to do so. This explains the high variance and noticeable steps shown in Figure \ref{fig:diresults}.
	
	The next section shows that an improperly-defined set of human-provided options, as may happen in design phase, still allows the agent to perform reasonably well. Combined with our results with random \Next, this shows that \Next can be tailored to the exact amount of domain knowledge available for a particular task.
	
	\subsection{TreeMaze}
	\label{sec:treemaze}
	
	The optimal set of options and \Next may be difficult to design. When the agent learns the option policies, the previous section demonstrates that random \Next suffice. This experiment focuses on human-provided option policies, and shows that a sub-optimal set of options, arising from a mis-specification of the environment or normal trial-and-error in design phase, does not prevent agents with \Next from learning reasonably good policies.

	TreeMaze is our generalization of the T-maze environment \cite{Bakker2001} to arbitrary heights. The agent starts at the root of the tree-like maze depicted in Figure \ref{fig:treemaze}, and has to reach the extremity of one of the 8 leaves. The leaf to be reached (the goal) is chosen uniformly randomly before each episode, and is indicated to the agent using 3 bits, observed one at a time during the first 3 time-steps. The agent receives no bit afterwards, and has to remember them in order to navigate to the goal. The agent observes its position in the current corridor (0 to 4) and the number of T junctions it has already crossed (0 to 3). A reward of -0.1 is given each time-step, +10 when reaching the goal. The episode finishes when the agent reaches any of the leaves. The optimal reward is 8.2.
	
	\begin{figure}
		\centering
		\includegraphics{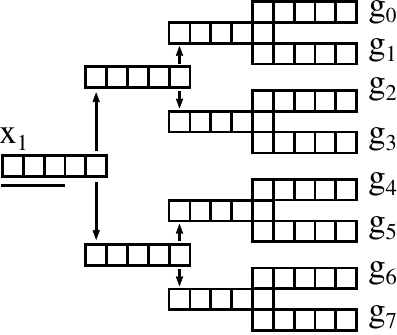}
		\caption{TreeMaze environment. The agent starts at $x_1$ and must go to one of the leaves. The leaf to be reached is indicated by 3 bits observed at time-steps 1, 2 and 3.}
		\label{fig:treemaze}
	\end{figure}
	
	We consider 14 options with predefined memoryless policies, several of them sharing the same policy, but encoding distinct states (among 14) of a 3-bit memory where some bits may be unknown. 6 partial-knowledge options $\omega_{0\md\md}$, $\omega_{1\md\md}$, $\omega_{00\md}$, ..., $\omega_{11\md}$ go right then terminate. 8 full-knowledge options $\omega_{000}$, $\omega_{001}$, ..., $\omega_{111}$ go to their corresponding leaf. \Next are defined so that any option may only be followed by itself, or one that represents a memory state where a single 0 or - has been flipped to 1. Five agents have to learn their top-level policy, which requires them to learn how to use the available options to remember to which leaf to go. The agents do not know the name or meaning of the options. Three agents have access to all 14 options (with, without \Next, and LSTM). The agent with \Next (8) only has access to full-knowledge options, and therefore cannot disambiguate unknown and 0 bits. The agent with \Next (4) is restricted to options $\omega_{000}$, $\omega_{010}$, $\omega_{100}$ and $\omega_{110}$ and therefore cannot reach odd-numbered goals. The options of the (8) and (4) agents terminate in the first two cells of the first corridor, to allow the top-level policy to observe the second and third bits.
	
	Figure \ref{fig:treemazeresults} shows that the agent with \Next (14) consistently learns the optimal policy for this task. When the number of options is reduced, the quality of the resulting policies decreases, while still remaining above the agent without \Next. Even the agent with 4 options, that cannot reach half the goals, performs better than the agent without \Next but 14 options. This experiment demonstrates that \Next provide measurable benefits over standard initiation sets, even if the option set is largely reduced. 
	
	Combined, our three experiments demonstrate that \Next lead to optimal policies in challenging POMDPs, consistently outperform LSTM over options, allow the option policies to be learned, and can still be used when reduced or no domain knowledge is available.
	
	\begin{figure}
		\centering
		\includegraphics{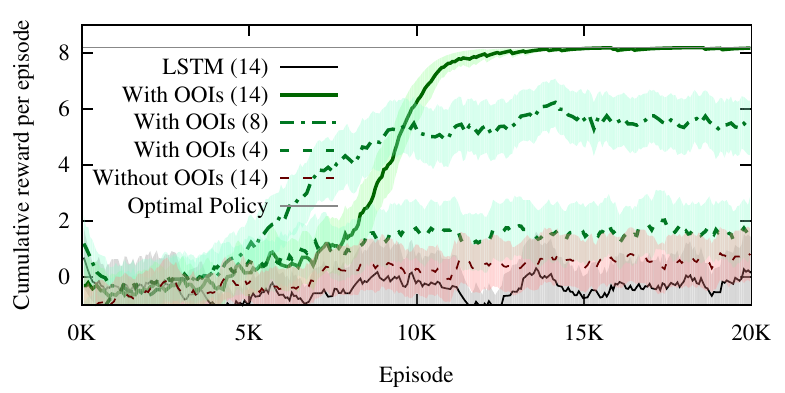}
		\caption{Cumulative reward per episode obtained on TreeMaze, using 14, 8 or 4 options. Even with an insufficient amount of options (8 or 4), \Next lead to better performance than no \Next but 14 options. LSTM over options learns the task after more than 100K episodes.}
		\label{fig:treemazeresults}
	\end{figure}
	
	\section{Conclusion and Future Work}
	\label{sec:conclusion}
	
	This paper proposes \Next, an extension of the initiation sets of options so that they restrict which options are allowed to be executed after one terminates. This makes options as expressive as Finite State Controllers. Experimental results confirm that challenging partially observable tasks, simulated or on physical robots, one of them requiring exact information storage for hundreds of time-steps, can now be solved using options. Our experiments also illustrate how \Next lead to reasonably good policies when the option set is improperly defined, and that learning the option policies allow random \Next to be used, thereby providing a turnkey solution to partial observability.
	
	Options with \Next also perform surprisingly well compared to an LSTM network over options. While LSTM over options does not require the design of \Next, their ability to learn without any a-priori knowledge comes at the cost of sample efficiency and explainability. Furthermore, random \Next are as easy to use as an LSTM and lead to superior results (see Section \ref{sec:duplicatedinput}). \Next therefore provide a compelling alternative to recurrent neural networks over options, applicable to a wide range of problems.
	
	Finally, the compatibility between \Next and a large variety of reinforcement learning algorithms leads to many future research opportunities. For instance, we have obtained very encouraging results in continuous action spaces, using CACLA \cite{VanHasselt2007} to implement parametric options, that take continuous arguments when executed, in continuous-action hierarchical POMDPs.
	
	\section*{Acknowledgments}
	
	The first author is ``Aspirant'' with the Science Foundation of Flanders (FWO, Belgium), grant number 1129317N. The second author is ``Postdoctoral Fellow'' with the FWO, grant number 12J0617N.
	
	Thanks to Finn Lattimore, who \emph{gave} a computer to the first author, so that he could finish this paper while attending the UAI 2017 conference in Sydney, after his own computer unexpectedly fried. Thanks to Joris Scharpff for his very helpful input on this paper.

	\vfill
	\pagebreak
	
	\small
	\bibliographystyle{named}
	\bibliography{biblio}

\end{document}